\documentclass{opt2025} 

\usepackage{multirow}
\usepackage{adjustbox}
\usepackage{framed}
\usepackage{wrapfig}
\usepackage{booktabs}

\newtheorem{assumption}{Assumption}

\title[Z-Score Gradient Filtering for SAM]{Sharpness-Aware Minimization with Z-Score Gradient Filtering}

\optauthor{%
\Name{Vincent-Daniel Yun} \Email{juyoung.yun@usc.edu}\\
\addr University of Southern California, USA \\
\addr Open Neural Network Research Lab, MODULABS, Republic of Korea}

\begin{document}

\maketitle

\begin{abstract}%
Deep neural networks achieve high performance across many domains but can still face challenges in generalization when optimization is influenced by small or noisy gradient components. Sharpness-Aware Minimization improves generalization by perturbing parameters toward directions of high curvature, but it uses the entire gradient vector, which means that small or noisy components may affect the ascent step and cause the optimizer to miss optimal solutions. We propose Z-Score Filtered Sharpness-Aware Minimization, which applies Z-score based filtering to gradients in each layer. Instead of using all gradient components, a mask is constructed to retain only the top percentile with the largest absolute Z-scores. The percentile threshold $Q_p$ determines how many components are kept, so that the ascent step focuses on directions that stand out most compared to the average of the layer. This selective perturbation refines the search toward flatter minima while reducing the influence of less significant gradients. Experiments on CIFAR-10, CIFAR-100, and Tiny-ImageNet with architectures including ResNet, VGG, and Vision Transformers show that the proposed method consistently improves test accuracy compared to Sharpness-Aware Minimization and its variants. The code repository is available at: \url{https://github.com/YUNBLAK/Sharpness-Aware-Minimization-with-Z-Score-Gradient-Filtering}

\end{abstract}

\section{Introduction}

Deep neural networks (DNNs)~\cite{goodfellow2016deep,dl1,dl2} achieve strong performance in tasks such as image classification~\cite{krizhevsky2012imagenet,he2016deep,krizhevsky2009learning}, speech recognition~\cite{hinton2012deep,speech1,speech2}, and natural language understanding~\cite{vaswani2017attention,nlp1,nlp2}. They are typically trained by minimizing empirical loss using optimizers such as SGD and Adam~\cite{robbins1951stochastic,sgd,kingma2015adam,adamw}.

\begin{wrapfigure}{l}{0.35\textwidth}
    \centering
    \includegraphics[width=\linewidth]{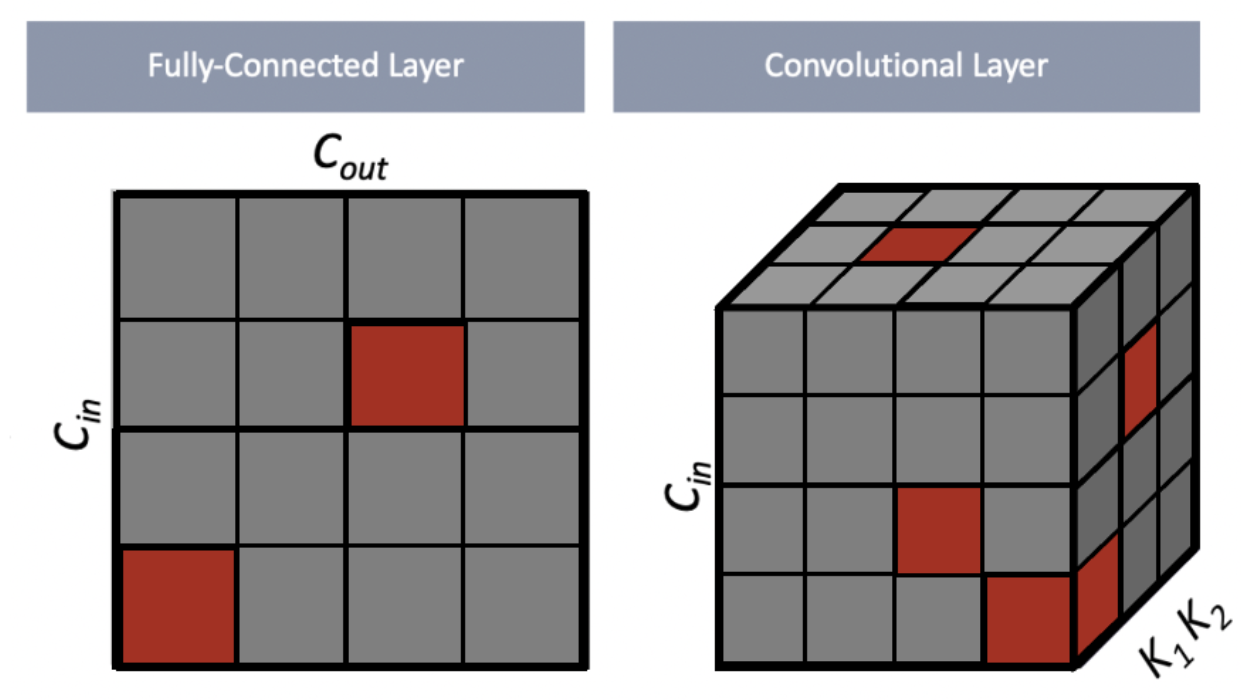}
    \vspace{-20pt}
    \caption{Ascent-step gradients after Z-score filtering.}
    \label{fig:intro2}
\end{wrapfigure}

Despite their success, DNNs often overfit~\cite{overfit1,overfit2,overfit3}, and poor generalization is frequently attributed to convergence toward sharp minima—regions of high curvature where small perturbations cause large increases in loss~\cite{hochreiter1997flat,keskar2017large,foret2021sharpness}. This issue becomes more pronounced in large models, where only a small subset of parameter directions meaningfully contributes to the curvature of the loss landscape~\cite{general1,general2,general3}. Sharpness-Aware Minimization (SAM) addresses this by perturbing parameters in the gradient direction and minimizing the worst-case loss within an $\ell_2$ ball~\cite{foret2021sharpness,kwon2021asam}.

\paragraph{Existing Problem.}
SAM constructs its ascent direction using the \emph{entire} gradient vector, including many small-magnitude components that provide little curvature information and largely reflect minibatch noise. Because the gradient is normalized before the perturbation is applied, even tiny noisy coordinates influence the ascent direction as much as large, informative ones. This distorts the perturbation, misaligns it with sharp curvature directions, and amplifies stochasticity in high-dimensional parameter spaces. In contrast, large gradient components typically correspond to directions of meaningful curvature, and focusing on them yields a more stable and faithful estimation of sharpness that better reflects the local geometry of the loss.

\paragraph{Our idea.}
We propose \textit{Z-Score Filtered Sharpness-Aware Minimization (ZSharp)}, which applies layer-wise Z-score normalization~\cite{yun2024znorm} to identify statistically significant gradient components and retains only the top percentile (e.g., top 5\%) for constructing the ascent direction. The percentile threshold \(Q_p\) controls the fraction \((1 - Q_p)\) of coordinates kept. Unlike ASAM~\cite{kwon2021asam} and Friendly-SAM~\cite{kim2022friendly}, ZSharp requires only a single hyperparameter and remains compatible with various optimizers and architectures.

We evaluate ZSharp on CIFAR-10~\cite{krizhevsky2009learning}, CIFAR-100~\cite{krizhevsky2009learning}, and Tiny-ImageNet~\cite{tinyimagenet} with models including ResNet~\cite{he2016deep}, VGG~\cite{simonyan2015very}, and Vision Transformers~\cite{dosovitskiy2020image}. Across all settings, ZSharp consistently matches or outperforms SAM and its variants, demonstrating that selectively filtering out small gradients leads to more stable ascent directions and better generalization.

\section{Methodology}

We propose Z-Score Filtered Sharpness-Aware Minimization \textit{(ZSharp)}, a method that improves neural network training by using Z-score normalization~\cite{yun2024znorm} and filtering in a Sharpness-Aware framework~\cite{foret2021sharpness}. When using the full gradient in the ascent step, small and noisy gradient components can weaken important curvature directions and may cause the optimizer to miss optimal convergence points. ZSharp mitigates this by retaining only the larger gradient components in each layer during the ascent step, which reduces the influence of noise.  \\

\noindent\textbf{Preliminaries.} 
We consider a supervised learning framework with a dataset \(\mathcal{D} = \{(\mathbf{x}_i, y_i)\}_{i=1}^N\), where \(\mathbf{x}_i \in \mathbb{R}^m\) denotes input features and \(y_i \in \mathcal{Y}\) represents labels. The neural network, parameterized by weights \(w \in \mathbb{R}^d\), defines a mapping \(f: \mathbb{R}^m \times \mathbb{R}^d \to \mathcal{Y}\). For \(L\) layers, the \(\ell\)-th layer’s parameters are \(w^{(\ell)} \in \mathbb{R}^{d_\ell}\), with \(\sum_{\ell=1}^L d_\ell = d\). The empirical loss \(L(w) = \frac{1}{N} \sum_{i=1}^N \ell(f(\mathbf{x}_i; w), y_i)\) has gradient \(\nabla L(w) \in \mathbb{R}^d\), with \(\ell_2\)-norm \(\|\nabla L(w)\|_2\). The percentile threshold \(Q_p\) controls the proportion of gradient components retained after filtering, where \((1 - Q_p)\) denotes the fraction of components with the largest magnitudes that are kept for the ascent step.

\subsection{Sharpness-Aware Minimization}
In standard SAM~\cite{foret2021sharpness, sam1, sam2}, the procedure consists of: (i) \emph{Ascent:} perturbing parameters in the direction that increases the loss the most. 
\begin{equation}
\epsilon_{\text{SAM}} = \rho \cdot \frac{\nabla L(w)}{\|\nabla L(w)\|_2 + \delta}, \quad \tilde{w} = w + \epsilon_{\text{SAM}},
\end{equation}
(ii) \emph{Minimization:} compute the gradient at the perturbed point, $g = \nabla L(\tilde{w})$,  
(iii) \emph{Weight update:} apply a base optimizer $O$ (e.g., SGD, Adam) to update parameters,  
\begin{equation}
w \leftarrow w - \eta\, O(g).
\end{equation}
ZSharp keeps steps (ii) and (iii) identical to SAM, but replaces $\nabla L(w)$ in the ascent step with the filtered gradient $\nabla L(w)_{\Omega}$, focusing the perturbation on statistically significant directions.


\subsection{Z-Score Filtered Ascent Step} \label{zsharp}
ZSharp modifies the ascent step of SAM~\cite{foret2021sharpness, sam1, sam2} 
by applying Layer-wise Z-Score Normalization~\cite{yun2024znorm} and retaining 
only the top $(1-Q_p)$ fraction of components by absolute Z-score.

Let $\{g^{(\ell)}\}_{\ell=1}^L$ denote the gradient vectors for each layer $\ell$, 
with $g^{(\ell)} \in \mathbb{R}^{d_\ell}$. 
For each layer, we define the Z-score normalized gradient as $\Omega(g^{(\ell)})_i = 
\frac{g^{(\ell)}_i - \mu(g^{(\ell)})}{\sigma(g^{(\ell)})}, 
\quad i = 1, \ldots, d_\ell$ where $\mu(g^{(\ell)}) = \frac{1}{d_\ell}\sum_{i=1}^{d_\ell} g^{(\ell)}_i$ and
$\sigma(g^{(\ell)}) = \sqrt{\frac{1}{d_\ell}\sum_{i=1}^{d_\ell} 
\bigl(g^{(\ell)}_i - \mu(g^{(\ell)})\bigr)^2}$
This ensures that each layer’s gradient is centered and rescaled independently, 
emphasizing components that deviate most relative to the layer statistics.

Then, we have $\Omega(\nabla L(w))$, layer-wise Z-score normalized gradients of the loss $L$ at parameters $w$.  
We then define a binary mask $\mathbf{m} \in \{0,1\}^d$ as
\begin{equation}
m_j = 
\begin{cases}
1 & \text{if } |\Omega(\nabla L(w))_j| > q_{Q_p}, \\
0 & \text{otherwise},
\end{cases}
\end{equation}
where $q_{Q_p}$ is the $Q_p$-th percentile of $|\Omega(\nabla L(w))|$.  
The filtered gradient is $\nabla L(w)_{\Omega} = \nabla L(w) \odot \mathbf{m}$, and the perturbation is computed as:
\begin{equation}
\epsilon =
\begin{cases}
\rho \cdot \frac{\nabla L(w)_{\Omega}}{\|\nabla L(w)_{\Omega}\|_2 + \delta} & \|\nabla L(w)_{\Omega}\|_2 > 0, \\
\rho \cdot \frac{\nabla L(w)}{\|\nabla L(w)\|_2 + \delta} & \text{otherwise},
\end{cases}
\end{equation}
where $\rho>0$ is the perturbation radius and $\delta=10^{-8}$ ensures numerical stability. The ascent update is $\tilde{w} = w + \epsilon.$

\section{Theoretical Analysis}
We adapt the standard SAM convergence proof~\cite{foret2021sharpness,sam_proof,sam2} to the ZSharp setting, where the ascent step uses the Z-score filtered gradient. Theorem~\ref{theorem5} establishes convergence under the same smoothness and bounded-variance assumptions as SAM~\cite{foret2021sharpness,sam_proof,sam2}. Detailed proofs are in Appendix~\ref{convergence}.

\begin{assumption}[$\beta$-smoothness]
\label{append:smoothness}
The loss function \(L:\mathbb{R}^d \to \mathbb{R}\) is \(\beta\)-smooth, meaning its gradient is \(\beta\)-Lipschitz continuous, $\|\nabla L(u) - \nabla L(v)\| \le \beta \|u - v\|, \quad \forall\, u,v \in \mathbb{R}^d.$ Equivalently, for any \(u,v\), $L(u) \le L(v) + \langle \nabla L(v),\, u - v \rangle + \tfrac{\beta}{2}\|u - v\|^2.$
\end{assumption}

\begin{assumption}[Unbiased stochastic gradient]
\label{append:unbiased}
At each iteration \(t\), the stochastic gradient provides an unbiased estimate of the true gradient: $\mathbb{E}[\nabla L_t(w_t)] = \nabla L(w_t).$
\end{assumption}

\begin{lemma}
Given a $\beta$-smooth loss function $L(x)$, 
the following bound holds:
\begin{equation}
\langle \nabla L(u) - \nabla L(v), u - v \rangle 
\ge -\beta \|u - v\|^2.
\end{equation}
\end{lemma}

\begin{lemma}
\label{lemma:A2}
Let \(L\) be a \(\beta\)-smooth loss function. At iteration \(t\), let 
\(\nabla L_t(w_t)_{\Omega}\) be an unbiased stochastic estimator of 
\(\nabla L(w_t)_{\Omega}\) with bounded variance $\mathbb{E}\!\left[\left\|\nabla L_t(w_t)_{\Omega} - \nabla L(w_t)_{\Omega}\right\|^2\right] \le \frac{\sigma_{\Omega}^2}{b}$ Then, for any \(r>0\),
\begin{align}
\mathbb{E}[
\left\langle
\nabla  L(w_t + r\,\nabla L_t(w_t)_{\Omega}),\nabla L(w_t)\right\rangle] \ge \frac{1}{2}\,\mathbb{E}\!\left[\|\nabla L(w_t)\|^2\right] - \frac{\beta^2 r^2}{2}\,\mathbb{E}\!\left[\|\nabla L(w_t)_{\Omega}\|^2\right]
- \frac{\beta^2 r^2\sigma_{\Omega}^2}{2b}.
\end{align}
\end{lemma}

\begin{lemma}
Assume that $L$ is $\beta$-smooth and the filtered stochastic gradient 
$\nabla L_t(w_t)_{\Omega}$ satisfies the variance bound
$\mathbb{E}\!\left[\|\nabla L_t(w_t)_{\Omega} - \nabla L(w_t)_{\Omega}\|^2\right] 
\le \frac{\sigma_{\Omega}^2}{b}.$ Consider the ZSharp-SAM updates $w_{t+1/2} = w_t + r \,\nabla L_t(w_t)_{\Omega},$ and $w_{t+1}   = w_{t+1/2} - \eta \,\nabla L_t(w_{t+1/2}).$ If the step size satisfies $\eta \le \tfrac{1}{4\beta}$ and the ascent radius satisfies 
$\beta^2 r^2 \le \tfrac{1}{4}$, then
\begin{align}
\mathbb{E}[L(w_{t+1})]
\le 
\mathbb{E}[L(w_t)]
- \frac{\eta}{4}\,\mathbb{E}\big[\|\nabla L(w_t)\|^2\big]
+ \frac{2\eta\beta^2 r^2}{b}\,\sigma_{\Omega}^2
+ \frac{\eta^2 \beta}{b}\,\sigma_{\Omega}^2.
\end{align}
\end{lemma}

\begin{theorem}
Assume that $L$ is $\beta$-smooth and that the filtered stochastic gradient
$\nabla L_t(w_t)_{\Omega}$ satisfies the variance bound $\mathbb{E}\!\left[\|\nabla L_t(w_t)_{\Omega} - \nabla L(w_t)_{\Omega}\|^2\right] le \frac{\sigma_{\Omega}^2}{b}.$
All expectations are taken over the mini-batch at iteration $t$.
Suppose the ZSharp-SAM updates $w_{t+1/2} = w_t + r \,\nabla L_t(w_t)_{\Omega} and $$w_{t+1}= w_{t+1/2} - \eta \,\nabla L_t(w_{t+1/2})$ use step size $\eta$ and ascent radius $r$ satisfying
$\eta \le \tfrac{1}{4\beta}$ and $\beta^2 r^2 \le \tfrac{1}{4}$.
Then ZSharp-SAM satisfies
\begin{align}
\frac{1}{T} \sum_{t=0}^{T-1} \mathbb{E} \big[ \|\nabla L(w_t)\|^2 \big]
\le \frac{4}{T \eta} \big( L(w_0) - \mathbb{E}[L(w_T)] \big)
+ \frac{8\beta^2 r^2}{b}\,\sigma_{\Omega}^2
+ \frac{4\eta \beta}{b}\,\sigma_{\Omega}^2.
\end{align}
\label{theorem5}
\end{theorem}

\begin{corollary}[Convergence with diminishing step sizes]
Assume that $L$ is $\beta$-smooth, bounded below by $L_{\inf}$, and $\nabla L_t(w_t)_{\Omega}$ satisfies $\mathbb{E}\!\left[\|\nabla L_t(w_t)_{\Omega} - \nabla L(w_t)_{\Omega}\|^2\right]
$$\le \frac{\sigma_{\Omega}^2}{b}.$
Consider the ZSharp-SAM updates $w_{t+1/2} = w_t + r_t \,\nabla L_t(w_t)_{\Omega}, $ and $w_{t+1}   = w_{t+1/2} - \eta_t \,\nabla L_t(w_{t+1/2}),$ with step sizes $\{\eta_t\}_{t\ge 0}$ and ascent radii $\{r_t\}_{t\ge 0}$ satisfying
$\eta_t \le \frac{1}{4\beta} $$ \beta^2 r_t^2 \le \frac{1}{4} \text{for all } t,$
$\sum_{t=0}^\infty \eta_t = \infty, $$\sum_{t=0}^\infty \eta_t^2 < \infty, $$\sum_{t=0}^\infty \eta_t r_t^2 < \infty.$
Then
\begin{align}
\sum_{t=0}^\infty \eta_t\,\mathbb{E}\big[\|\nabla L(w_t)\|^2\big] < \infty, \quad \text{and in particular} \quad \liminf_{t\to\infty} \mathbb{E}\big[\|\nabla L(w_t)\|^2\big] = 0.
\end{align}
\label{corollary}
\end{corollary}

This corollary establishes that ZSharp retains the convergence guarantees of standard SAM: under the same smoothness and bounded-variance assumptions, Z-score filtering in the ascent step still ensures convergence to stationary points. In other words, applying Z-score gradient filtering does not hinder the convergence behavior of SAM.

\section{Experimental Results}

To evaluate the effectiveness of ZSharp, we compare it with the standard SAM~\cite{foret2021sharpness} and its variants, ASAM~\cite{kwon2021asam} and Friendly-SAM~\cite{kim2022friendly}, as well as the baseline optimizer. \\

\noindent \textbf{Experimental Settings.} 
We evaluate ZSharp on CIFAR-10/100~\cite{krizhevsky2009learning} and Tiny-ImageNet~\cite{tinyimagenet} using ResNet-56/110~\cite{he2016deep}, VGG16\_BN~\cite{simonyan2015very}, and ViT models~\cite{dosovitskiy2020image}. 
All models are trained for 200 epochs with batch size 256 using AdamW~\cite{kingma2015adam, adamw} (lr=0.001, weight decay $5\times10^{-5}$) and step decay (0.75 every 10 epochs). 
For SAM~\cite{foret2021sharpness}, ASAM~\cite{kwon2021asam}, Friendly-SAM~\cite{kim2022friendly}, and ZSharp, we set $\rho=0.05$ following prior work. 
ZSharp applies Z-score filtering ($Q_p=0.95$) in the ascent step, keeping the top 5\% of gradient components. All experiments run on a single RTX 4090 GPU and results are averaged over 3 seeds. We used the publicly available implementations of ViT training~\cite{vit-cifar}, ASAM~\cite{kwon2021asam}, and FSAM~\cite{kim2022friendly} from open github repositories. \\

\noindent \textbf{Results.}
Table~\ref{table:exp_all} presents the Top-1 test accuracy and train loss across five architectures (ResNet-56, ResNet-110, VGG-16/BN, ViT-7/8/8-384, and ViT-7/8/12-768) on CIFAR-10, CIFAR-100, and Tiny-ImageNet. 
Overall, ZSharp consistently achieves the highest or comparable test accuracy among all methods, with gains observed across both convolution-based and transformer-based architectures.

\begin{table*}[h]
\centering
\tiny
\begin{adjustbox}{width=1\textwidth}
\begin{tabular}{llllllll} 
\toprule
& & \multicolumn{2}{c}{\textbf{CIFAR-10~\cite{krizhevsky2009learning}}} & \multicolumn{2}{c}{\textbf{CIFAR-100~\cite{krizhevsky2009learning}}} & \multicolumn{2}{c}{\textbf{Tiny-ImageNet~\cite{tinyimagenet}}} \\
Network & Method & Top-1 Test Acc. & Train Loss. & Top-1 Test Acc. & Train Loss. & Top-1 Test Acc. & Train Loss. \\
\midrule
\multirow{5}{*}{ResNet-56~\cite{he2016deep}}  & AdamW (Baseline)~\cite{adamw}                   & 0.9108 ± 0.0045 & 0.0057 ± 0.0013               & 0.6420 ± 0.0031 & 0.0452 ± 0.0629              & 0.4747 ± 0.0031 & 0.0121 ± 0.0105  \\
                                        & SAM~\cite{foret2021sharpness}                         & 0.9160 ± 0.0021 & 0.0221 ± 0.0051               & 0.6527 ± 0.0025 & 0.1097 ± 0.0584              & 0.4938 ± 0.0026 & 0.0453 ± 0.0127  \\
                                        & ASAM~\cite{kwon2021asam}                              & 0.9228 ± 0.0034 & 0.0366 ± 0.0093               & 0.6646 ± 0.0017 & 0.1952 ± 0.0419              & 0.5072 ± 0.0012 & 0.0564 ± 0.0091  \\
                                        & Friendly-SAM~\cite{kim2022friendly}                   & 0.9179 ± 0.0025 & 0.0219 ± 0.0076               & 0.6549 ± 0.0024 & 0.1051 ± 0.0417              & 0.4948 ± 0.0023 & 0.0444 ± 0.0102  \\
                                        & ZSharp (Ours)                                         & \textbf{0.9264 ± 0.0032} & 0.0630 ± 0.0064      & \textbf{0.6679 ± 0.0015} & 0.2510 ± 0.0438     & \textbf{0.5073 ± 0.0014} & 0.0828 ± 0.0129  \\
\midrule
\multirow{5}{*}{ResNet-110~\cite{he2016deep}}  & AdamW (Baseline)~\cite{adamw}          & 0.9140 ± 0.0031 & 0.0056 ± 0.0023                    & 0.6650 ± 0.0025 & 0.0149 ± 0.0059              & 0.4878 ± 0.0045 & 0.0556 ± 0.0114  \\
                            & SAM~\cite{foret2021sharpness}                             & 0.9233 ± 0.0025 & 0.0188 ± 0.0037                    & 0.6815 ± 0.0019 & 0.0531 ± 0.0121              & 0.5005 ± 0.0045 & 0.1417 ± 0.0216  \\
                            & ASAM~\cite{kwon2021asam}                                  & 0.9261 ± 0.0023 & 0.0288 ± 0.0056                    & 0.6796 ± 0.0036 & 0.0915 ± 0.0123              & 0.5105 ± 0.0045 & 0.2894 ± 0.0241  \\
                            & Friendly-SAM~\cite{kim2022friendly}                       & 0.9193 ± 0.0013 & 0.0190 ± 0.0036                    & 0.6762 ± 0.0021 & 0.0524 ± 0.0113              & 0.5027 ± 0.0045 & 0.1402 ± 0.0091  \\
                            & ZSharp (Ours)                                             & \textbf{0.9293 ± 0.0017} & 0.0618 ± 0.0097           & \textbf{0.6844 ± 0.0023} & 0.1656 ± 0.0213     & \textbf{0.5207 ± 0.0045} & 0.4137 ± 0.0311  \\
\midrule
\multirow{5}{*}{VGG-16/BN~\cite{simonyan2015very}}  & AdamW (Baseline)~\cite{adamw}             & 0.9247 ± 0.0013 & 0.0058 ± 0.0031               & 0.6999 ± 0.0102 & 0.0092 ± 0.0051              & 0.5507 ± 0.0093 & 0.0043 ± 0.0071   \\
                            & SAM~\cite{foret2021sharpness}                                     & 0.9337 ± 0.0018 & 0.0171 ± 0.0093               & 0.7092 ± 0.0093 & 0.0139 ± 0.0073              & 0.5587 ± 0.0103 & 0.0363 ± 0.0183   \\
                            & ASAM~\cite{kwon2021asam}                                          & \textbf{0.9355 ± 0.0012} & 0.0237 ± 0.0047      & 0.7170 ± 0.0121 & 0.0375 ± 0.0118              & 0.5647 ± 0.0191 & 0.0644 ± 0.0237   \\
                            & Friendly-SAM~\cite{kim2022friendly}                               & 0.9290 ± 0.0017 & 0.0163 ± 0.0093               & 0.7099 ± 0.0083 & 0.0495 ± 0.0125              & 0.5544 ± 0.0204 & 0.0349 ± 0.0153   \\
                            & ZSharp (Ours)                                                     & 0.9327 ± 0.0020 & 0.0351 ± 0.0144               & \textbf{0.7207 ± 0.0071} & 0.0375 ± 0.0137     & \textbf{0.5673 ± 0.0231} & 0.1248 ± 0.0351   \\
\midrule
\multirow{5}{*}{ViT-7/8/8-384~\cite{dosovitskiy2020image}}  & AdamW (Baseline)~\cite{adamw}     & 0.8398 ± 0.0028 & 0.0087 ± 0.0092               & 0.5479 ± 0.0011 & 0.0042 ± 0.0031               & 0.2843 ± 0.0012 & 0.0056 ± 0.0014  \\
                            & SAM~\cite{foret2021sharpness}                                     & 0.8432 ± 0.0032 & 0.0273 ± 0.0101               & 0.5557 ± 0.0013 & 0.0255 ± 0.0141               & 0.2897 ± 0.0009 & 0.0363 ± 0.0098  \\
                            & ASAM~\cite{kwon2021asam}                                          & 0.8302 ± 0.0034 & 0.0367 ± 0.0138               & 0.5566 ± 0.0031 & 0.0349 ± 0.0193               & 0.2522 ± 0.0032 & 0.0644 ± 0.0137  \\
                            & Friendly-SAM~\cite{kim2022friendly}                               & 0.8476 ± 0.0044 & 0.0273 ± 0.0093               & 0.5608 ± 0.0023 & 0.0228 ± 0.0138               & 0.3000 ± 0.0012 & 0.0349 ± 0.0083  \\
                            & ZSharp (Ours)                                                     & \textbf{0.8543 ± 0.0029} & 0.0647 ± 0.0216      & \textbf{0.5748 ± 0.0051} & 0.0730 ± 0.0212      & \textbf{0.3057 ± 0.0021} & 0.1248 ± 0.0413  \\

\midrule
\multirow{5}{*}{ViT-7/8/12-768~\cite{dosovitskiy2020image}}  & AdamW (Baseline)~\cite{adamw}      & 0.8438 ± 0.0021 & 0.0087 ± 0.0031               & 0.5615 ± 0.0013 & 0.0040 ± 0.0045               & 0.2991 ± 0.0010 & 0.0065 ± 0.0032  \\
                            & SAM~\cite{foret2021sharpness}                                         & 0.8486 ± 0.0018 & 0.0293 ± 0.0098               & 0.5691 ± 0.0014 & 0.0234 ± 0.0076               & 0.3014 ± 0.0015 & 0.0297 ± 0.0098  \\
                            & ASAM~\cite{kwon2021asam}                                              & 0.8395 ± 0.0020 & 0.0371 ± 0.0101               & 0.5649 ± 0.0027 & 0.0347 ± 0.0116               & 0.3023 ± 0.0008 & 0.0512 ± 0.0161  \\
                            & Friendly-SAM~\cite{kim2022friendly}                                   & 0.8525 ± 0.0021 & 0.0283 ± 0.0084               & 0.5655 ± 0.0021 & 0.0246 ± 0.0122               & 0.3034 ± 0.0013 & 0.0441 ± 0.0141  \\
                            & ZSharp (Ours)                                                         & \textbf{0.8586 ± 0.0023} & 0.0635 ± 0.0196      & \textbf{0.5777 ± 0.0031} & 0.0709 ± 0.0178      & \textbf{0.3104 ± 0.0019} & 0.1341 ± 0.0211  \\

\bottomrule
\end{tabular}
\end{adjustbox}
\caption{Top-1 Test Accuracy and Train Loss for ResNet-56~\cite{he2016deep}, ResNet-110~\cite{he2016deep}, VGG-16/BN~\cite{simonyan2015very}, ViT-7/8/8-384~\cite{dosovitskiy2020image}, and ViT-7/8/12-768~\cite{dosovitskiy2020image} on CIFAR-10~\cite{krizhevsky2009learning}, CIFAR-100~\cite{krizhevsky2009learning}, and Tiny-ImageNet datasets~\cite{tinyimagenet} across different SAM variants such as AdamW (Baseline)~\cite{adamw}, SAM~\cite{foret2021sharpness}, Friendly-SAM~\cite{kim2022friendly}, ASAM~\cite{kwon2021asam}, and ZSharp (Ours). For ViT models, ViT-7/8/8-384 and ViT-7/8/12-768 denote Vision Transformers with 7 layers, 8 attention heads, patch sizes of 8, and MLP dimensions of 384 and 768, respectively.}
\label{table:exp_all}
\end{table*}
\vspace{-0.8cm}
\begin{table*}[h]
\centering
\begin{adjustbox}{width=1\textwidth}
\begin{tabular}{lllll|lllll} 
\toprule
Network & Method & \(Q_p\) & Top-1 Test Acc. & Train Loss. & Network & Method & \(Q_p\) & Top-1 Test Acc. & Train Loss. \\
\midrule
\multirow{7}{*}{ResNet-56~\cite{he2016deep}}  & AdamW~\cite{adamw}               & N/A   & 0.9108 ± 0.0045 & 0.0057 ± 0.0013            & \multirow{7}{*}{ViT-7/8/8-384~\cite{dosovitskiy2020image}}  & AdamW~\cite{adamw}     & N/A  & 0.8398 ± 0.0028 & 0.0087 ± 0.0092\\
& SAM~\cite{foret2021sharpness}              & N/A   & 0.9160 ± 0.0021 & 0.0221 ± 0.0051            && SAM~\cite{foret2021sharpness}   & N/A   & 0.8432 ± 0.0032 & 0.0273 ± 0.0101\\
& ZSharp (Ours)                              & 0.95  & \textbf{0.9264 ± 0.0032} & 0.0630 ± 0.0064   && ZSharp (Ours) & 0.95  & \textbf{0.8543 ± 0.0029} & 0.0647 ± 0.0216\\
& ZSharp (Ours)                              & 0.90  & 0.9212 ± 0.0015 & 0.0710 ± 0.0061            && ZSharp (Ours) & 0.90  & 0.8482 ± 0.0031 & 0.0748 ± 0.0081\\
& ZSharp (Ours)                              & 0.85  & 0.9189 ± 0.0023 & 0.0679 ± 0.0067            && ZSharp (Ours) & 0.85  & 0.8424 ± 0.0043 & 0.0825 ± 0.0086\\
& ZSharp (Ours)                              & 0.80  & 0.9153 ± 0.0027 & 0.0731 ± 0.0053            && ZSharp (Ours) & 0.80  & 0.8421 ± 0.0038 & 0.0863 ± 0.0119\\
& ZSharp (Ours)                              & 0.75  & 0.9132 ± 0.0017 & 0.0789 ± 0.0079            && ZSharp (Ours) & 0.75  & 0.8378 ± 0.0027 & 0.0999 ± 0.0102\\
\bottomrule
\end{tabular}
\end{adjustbox}
\caption{Training hyperparameters and results for all experiments, with settings identical to those in the Experimental Settings section except for varying \(Q_p\) values.}
\label{table:training_settings}
\end{table*}

\begin{figure}
\centering
\includegraphics[width=1\columnwidth]{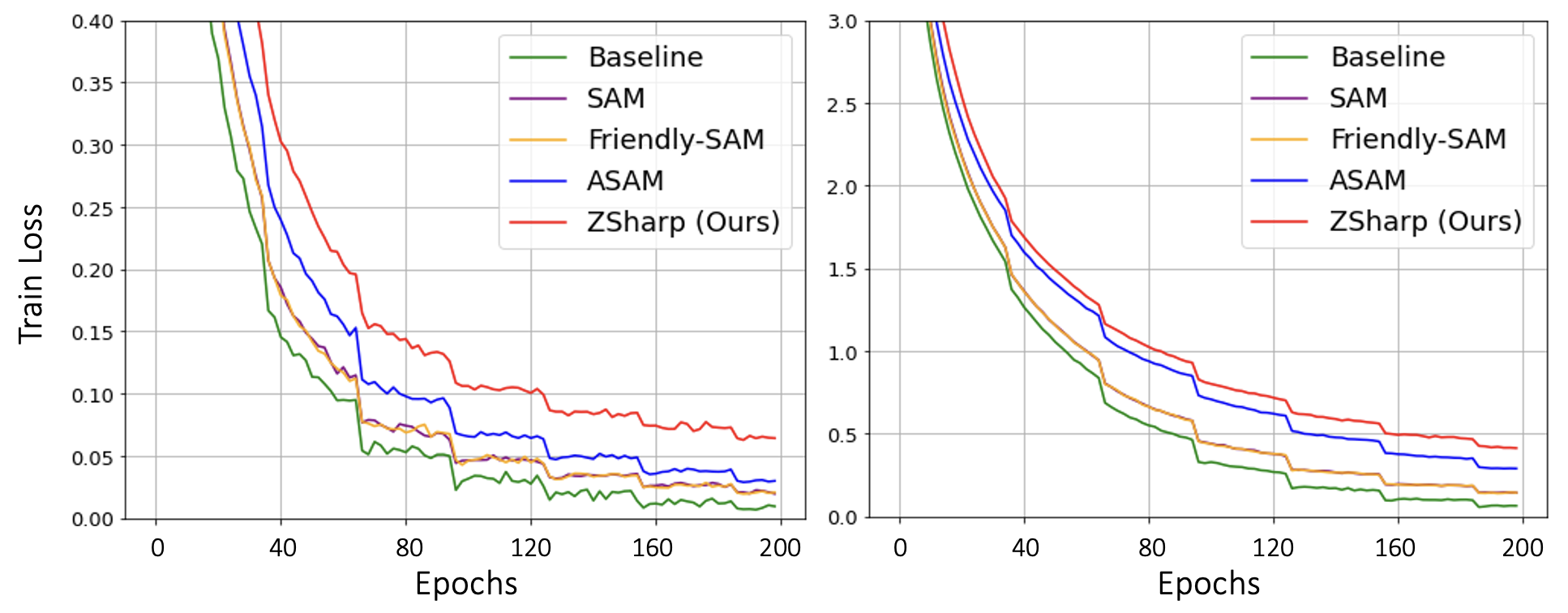}
\vspace{-12pt}
\caption{Train Loss comparison on CIFAR-10 for ResNet-56 (left) and ResNet-110 (right) across different SAM variants: Baseline, SAM, Friendly-SAM, ASAM, and ZSharp (Ours).}
\label{fig:trainloss}
\end{figure}

\noindent \textbf{Hyperparameter Tuning}
We analyze the percentile threshold \(Q_p\) in ZSharp, which retains the top \((1-Q_p)\%\) of gradients after Z-score filtering. Higher values (e.g., \(0.95\)) keep fewer but more significant components, while \(Q_p = 0.0\) recovers SAM. Table~\ref{table:training_settings} shows that \(Q_p = 0.95\) yields the best accuracy on both ResNet-56 and ViT-7/8/8-384, with performance gradually approaching SAM as \(Q_p\) decreases. We therefore use \(Q_p = 0.95\) for all subsequent experiments. \\

\noindent \textbf{Generalization Effect.} 
Figure~\ref{fig:trainloss} shows that ZSharp has a higher train loss but better test accuracy than other methods. Similar patterns have been observed in prior work~\cite{zhang2017understanding, bartlett2020benign, neyshabur2017exploring, keskar2017large}, where models with slightly higher training loss can generalize better when they converge to flatter or wider minima. 
This suggests that ZSharp’s selective focus on high-magnitude gradient components not only reduces overfitting but also helps the model find solutions with improved generalization performance on unseen data. \\

\section{Conclusion}
We proposed ZSharp, a sharpness-aware optimization method that applies  z-score gradient filtering to the ascent step of SAM, focusing updates on statistically significant gradient components. ZSharp preserves SAM’s convergence guarantees and consistently improves test accuracy over SAM and its variants across CIFAR-10, CIFAR-100, and Tiny-ImageNet on diverse architectures. With only one additional hyperparameter and no architectural changes, ZSharp offers an effective way to enhance generalization in deep
neural network training.

\section{Acknowledgement}
This research was supported by Brian Impact Foundation, a non-profit organization dedicated to the advancement of science and technology for all.

\bibliography{sample}
\newpage
\appendix
\renewcommand{\thetheorem}{\arabic{theorem}}
\renewcommand{\thelemma}{\arabic{lemma}}
\renewcommand{\theassumption}{\arabic{assumption}}
\setcounter{theorem}{0}
\setcounter{assumption}{0}

\begin{center}
{\LARGE \textbf{Appendix}}\\[1em]
\end{center}

\section{Overview}
\label{overview}
Figure~\ref{fig:overview} illustrates the overall process of ZSharp, highlighting how Z-score gradient filtering is integrated into the Sharpness-Aware Minimization (SAM) framework. 

\begin{figure*}[h]
\centering
\includegraphics[width=1\columnwidth]{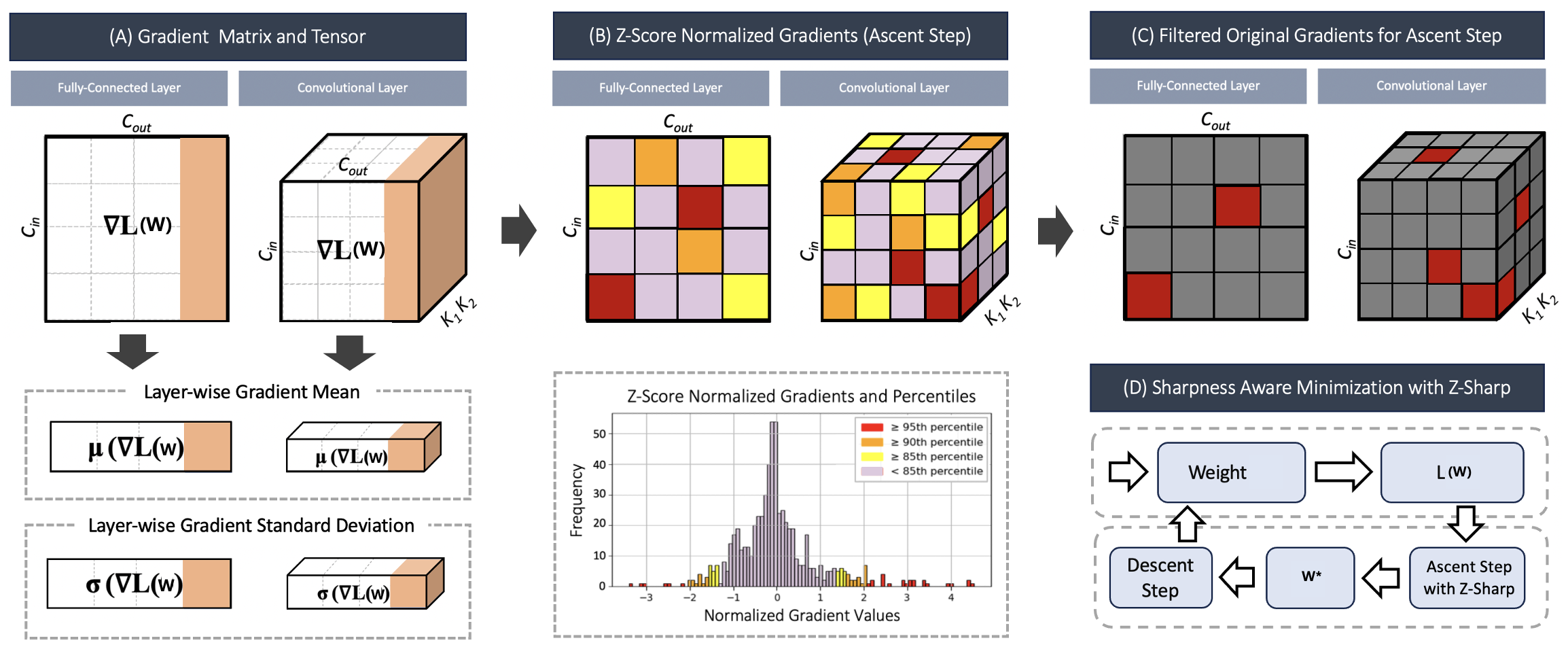}
\caption{
Overview of ZSharp: Z-Score Filtered Sharpness-Aware Minimization.
(A) Gradients from fully-connected and convolutional layers are used to compute layer-wise statistics. 
(B) Z-score normalization is applied to standardize gradients, followed by percentile-based filtering to select statistically significant components. 
(C) A binary mask retains only the top Z-score entries (e.g., top 5\%), filtering the gradient for the ascent step. 
(D) The filtered gradient is then used in the SAM ascent phase to refine the perturbation direction, enhancing generalization by focusing updates on curvature-sensitive directions.
}
\label{fig:overview}
\end{figure*}

\noindent\textbf{(A) Gradient Matrix and Tensor.} For each fully-connected and convolutional layer, we obtain the gradient tensor $\nabla L(w)$ during the ascent step of SAM. The layer-wise gradient mean $\mu(\nabla L(w))$ and standard deviation $\sigma(\nabla L(w))$ are computed to capture the statistical distribution of gradient values. \\

\noindent\textbf{(B-C) Z-Score Normalization and Gradient Filtering} Layer-wise Z-score normalization is applied to standardize the gradient values, producing normalized gradients $\Omega(\nabla L(w))$. A percentile-based ranking is then computed, where components are categorized (e.g., $\geq 95$th, $\geq 90$th, $\geq 85$th percentile, or below). A binary mask is generated to retain only the top $(1 - Q_p)\%$ of components with the largest absolute Z-scores. This mask is applied to the \emph{original} gradient (not the normalized one), resulting in a filtered gradient $\nabla L(w)_{\mathrm{\Omega}}$ that emphasizes statistically significant components. \\

\noindent\textbf{(D) Integration into SAM.} The filtered gradient is used in place of the original gradient in SAM’s ascent step to compute the perturbation $\epsilon$. The perturbed parameters $w^*$ are then used in the descent step with the base optimizer.\\

\section{Related Works}
\label{related}
Improving generalization in deep neural networks (DNNs)~\cite{goodfellow2016deep, dl1, dl2} has motivated many optimization strategies. Among these, normalization-based approaches and sharpness-aware optimization have been widely explored.  

Normalization techniques are effective in enhancing generalization performance. Batch Normalization~\cite{ioffe2015batch} and Layer Normalization~\cite{ba2016layer} act on activations, alleviating gradient vanishing while improving generalization. At the gradient level, gradient clipping~\cite{pascanu2013difficulty} limits gradient magnitude, and gradient centralization~\cite{yong2020gradient} subtracts mean values to improve convergence. Stochastic Gradient Sampling, as in StochGradAdam~\cite{stochgrad}, selects subsets of gradients during training, leading to stronger generalization particularly in ResNet-based CNNs. More recently, ZNorm~\cite{yun2024znorm} applies layer-wise Z-score normalization to gradients, providing consistent scaling and yielding enhanced generalization on benchmarks such as CIFAR-10 and in medical imaging tasks.

Beyond normalization, sharpness-aware optimization has emerged as a key framework, aiming to locate flatter minima that empirically correlate with stronger generalization. Sharpness-Aware Minimization (SAM)~\cite{foret2021sharpness,sam1,sam2,sam_proof} perturbs parameters in the gradient direction and minimizes the maximum loss within a local $\ell_2$ neighborhood. This approach improves generalization~\cite{chen2022when} compared to standard optimizers such as SGD~\cite{sgd} and Adam~\cite{adamw}. However, SAM constructs perturbations using the full gradient vector, including noisy or weak components, which can reduce precision in identifying sharpness-sensitive directions~\cite{zhuang2022gsam, ssam}.

Several extensions have been proposed. Adaptive SAM (ASAM)~\cite{kwon2021asam} rescales perturbations by curvature, improving robustness to parameter scaling. Friendly-SAM~\cite{kim2022friendly} approximates the sharpness objective to reduce computational cost, though sometimes at the expense of accuracy in architectures such as Vision Transformers (ViTs)~\cite{dosovitskiy2020image}. GSAM~\cite{zhuang2022gsam} aligns gradients to stabilize updates but requires additional hyperparameters.

ZSharp builds on SAM~\cite{foret2021sharpness} and ZNorm~\cite{yun2024znorm} by introducing statistical filtering into the perturbation step. During the ascent phase, gradients are first standardized within each layer using Z-score normalization, and these standardized values are used to compute a binary mask that identifies components above a given percentile threshold. This mask is then applied to the original gradients, retaining only the top $(1-Q_p)\%$ of components with the largest deviations from the mean. In this way, ZSharp reduces the influence of noise and small gradients in the ascent step, yielding sparse but targeted perturbations that better capture sharpness-related directions.

ZSharp introduces only one additional hyperparameter, the percentile threshold, without architectural or training modifications. It remains compatible with SAM implementations and base optimizers. Experiments on CIFAR-10~\cite{krizhevsky2009learning}, CIFAR-100~\cite{krizhevsky2009learning}, and Tiny-ImageNet~\cite{tinyimagenet} show improved generalization across ResNet~\cite{he2016deep}, VGG~\cite{simonyan2015very}, and ViT~\cite{dosovitskiy2020image}, suggesting that statistically guided filtering is an effective strategy for enhancing sharpness aware optimization in high dimensional or noisy gradient regimes.

\newpage
\section{Theoretical Analysis: Proof of Convergence}
\label{convergence}
We first present a couple of useful lemmas here. 
Our analysis borrows the proof structure used in~\cite{sam_proof, sam2}, 
but we adapt it to the ZSharp setting, where the ascent step uses 
the Z-score filtered gradient 
\(\nabla L(\cdot)_{\Omega}\) defined in Section~\ref{zsharp}, 
while the descent step still uses the original gradient \(\nabla L(\cdot)\).

\begin{assumption}[$\beta$-smoothness]
\label{append:smoothness}
The loss function \(L:\mathbb{R}^d \to \mathbb{R}\) is \(\beta\)-smooth, meaning its gradient is \(\beta\)-Lipschitz continuous, $\|\nabla L(u) - \nabla L(v)\| \le \beta \|u - v\|, \quad \forall\, u,v \in \mathbb{R}^d.$ Equivalently, for any \(u,v\), $L(u) \le L(v) + \langle \nabla L(v),\, u - v \rangle + \tfrac{\beta}{2}\|u - v\|^2.$
\end{assumption}

\begin{assumption}[Unbiased stochastic gradient]
\label{append:unbiased}
At each iteration \(t\), the stochastic gradient provides an unbiased estimate of the true gradient: $\mathbb{E}[\nabla L_t(w_t)] = \nabla L(w_t).$
\end{assumption}

\begin{lemma}
Given a $\beta$-smooth loss function $L(x)$, 
the following bound holds:
\begin{equation}
\langle \nabla L(u) - \nabla L(v), u - v \rangle 
\ge -\beta \|u - v\|^2.
\end{equation}
\end{lemma}
\begin{proof}
By $\beta$-smoothness of $L$, we have
\begin{align}
\|\nabla L(u) - \nabla L(v)\| 
&\le \beta \|u - v\|, 
\quad \forall \, u, v \in \mathbb{R}^d. \label{eq:smooth}
\end{align}
Multiplying both sides of \eqref{eq:smooth} by $\|v - u\|$ and using $\|u-v\| = \|v-u\|$, we get
\begin{align}
\|\nabla L(u) - \nabla L(v)\| \cdot \|v - u\| \le \beta \|u - v\| \cdot \|v - u\| \nonumber= \beta \|u - v\|^2. \label{eq:mult}
\end{align}
By the Cauchy–Schwarz inequality,
\begin{align}
\langle \nabla L(u) - \nabla L(v), v - u \rangle \le \|\nabla L(u) - \nabla L(v)\| \cdot \|v - u\| \nonumber\le \beta \|u - v\|^2.
\end{align}
Finally, multiplying above equation by $-1$ yields
\begin{align}
\langle \nabla L(u) - \nabla L(v), u - v \rangle
&\ge -\beta \|u - v\|^2.
\end{align}
\end{proof}

\begin{lemma}
\label{lemma:A2}
Let \(L\) be a \(\beta\)-smooth loss function. At iteration \(t\), let 
\(\nabla L_t(w_t)_{\Omega}\) be an unbiased stochastic estimator of 
\(\nabla L(w_t)_{\Omega}\) with bounded variance $\mathbb{E}\!\left[\left\|\nabla L_t(w_t)_{\Omega} - \nabla L(w_t)_{\Omega}\right\|^2\right] \le \frac{\sigma_{\Omega}^2}{b}$ Then, for any \(r>0\),
\begin{align}
\mathbb{E}[
\left\langle
\nabla  L(w_t + r\,\nabla L_t(w_t)_{\Omega}),\nabla L(w_t)\right\rangle] \ge \frac{1}{2}\,\mathbb{E}\!\left[\|\nabla L(w_t)\|^2\right] - \frac{\beta^2 r^2}{2}\,\mathbb{E}\!\left[\|\nabla L(w_t)_{\Omega}\|^2\right]
- \frac{\beta^2 r^2\sigma_{\Omega}^2}{2b}.
\end{align}
\end{lemma}

\begin{proof}
Define
\begin{align}
\Delta_t 
= \nabla L\!\left(w_t + r\,\nabla L_t(w_t)_{\Omega}\right)
 - \nabla L(w_t).
\end{align}
Then
\begin{align}
\left\langle 
\nabla L\!\left(w_t + r\,\nabla L_t(w_t)_{\Omega}\right),\;
\nabla L(w_t)
\right\rangle
= 
\|\nabla L(w_t)\|^2 + \langle \Delta_t, \nabla L(w_t) \rangle.
\end{align}
By Cauchy--Schwarz and Young’s inequality,
\begin{align}
\langle \Delta_t, \nabla L(w_t) \rangle
\ge 
-\frac{1}{2}\|\Delta_t\|^2 
-\frac{1}{2}\|\nabla L(w_t)\|^2.
\end{align}
Combining the two displays and taking expectations,
\begin{align}
\mathbb{E}\!\left[
\left\langle 
\nabla L\!\left(w_t + r\,\nabla L_t(w_t)_{\Omega}\right),\;
\nabla L(w_t)
\right\rangle
\right]
\ge
\frac{1}{2}\,\mathbb{E}\!\left[\|\nabla L(w_t)\|^2\right]
- \frac{1}{2}\,\mathbb{E}\!\left[\|\Delta_t\|^2\right].
\end{align}

By \(\beta\)-smoothness of \(L\),
\begin{align}
\|\Delta_t\|
&= \big\|\nabla L\!\left(w_t + r\,\nabla L_t(w_t)_{\Omega}\right)
 - \nabla L(w_t)\big\| \nonumber\\
&\le 
\beta r\,\|\nabla L_t(w_t)_{\Omega}\|.
\end{align}
Squaring and taking expectations,
\begin{align}
\mathbb{E}\!\left[\|\Delta_t\|^2\right]
\le 
\beta^2 r^2\,\mathbb{E}\!\left[\|\nabla L_t(w_t)_{\Omega}\|^2\right].
\end{align}
By variance decomposition,
\begin{align}
\mathbb{E}\!\left[\|\nabla L_t(w_t)_{\Omega}\|^2\right]
=
\mathbb{E}\!\left[\|\nabla L(w_t)_{\Omega}\|^2\right] 
+ \frac{\sigma_{\Omega}^2}{b}.
\end{align}
Therefore,
\begin{align}
\mathbb{E}\!\left[\|\Delta_t\|^2\right]
\le 
\beta^2 r^2\,\mathbb{E}\!\left[\|\nabla L(w_t)_{\Omega}\|^2\right]
+ \frac{\beta^2 r^2\sigma_{\Omega}^2}{b}.
\end{align}
Substituting this bound into the earlier inequality gives
\begin{align}
\mathbb{E}\!\left[
\left\langle 
\nabla L\!\left(w_t + r\,\nabla L_t(w_t)_{\Omega}\right),\;
\nabla L(w_t)
\right\rangle
\right]
\ge 
\frac{1}{2}\,\mathbb{E}\!\left[\|\nabla L(w_t)\|^2\right]
- \frac{\beta^2 r^2}{2}\,\mathbb{E}\!\left[\|\nabla L(w_t)_{\Omega}\|^2\right]
- \frac{\beta^2 r^2\sigma_{\Omega}^2}{2b},
\end{align}
which proves the claim.
\end{proof}

\begin{lemma}
Assume that $L$ is $\beta$-smooth and the filtered stochastic gradient 
$\nabla L_t(w_t)_{\Omega}$ satisfies the variance bound
$\mathbb{E}\!\left[\|\nabla L_t(w_t)_{\Omega} - \nabla L(w_t)_{\Omega}\|^2\right] 
\le \frac{\sigma_{\Omega}^2}{b}.$ Consider the ZSharp-SAM updates $w_{t+1/2} = w_t + r \,\nabla L_t(w_t)_{\Omega},$ and $w_{t+1}   = w_{t+1/2} - \eta \,\nabla L_t(w_{t+1/2}).$ If the step size satisfies $\eta \le \tfrac{1}{4\beta}$ and the ascent radius satisfies 
$\beta^2 r^2 \le \tfrac{1}{4}$, then
\begin{align}
\mathbb{E}[L(w_{t+1})]
\le 
\mathbb{E}[L(w_t)]
- \frac{\eta}{4}\,\mathbb{E}\big[\|\nabla L(w_t)\|^2\big]
+ \frac{2\eta\beta^2 r^2}{b}\,\sigma_{\Omega}^2
+ \frac{\eta^2 \beta}{b}\,\sigma_{\Omega}^2.
\end{align}
\end{lemma}

\begin{proof}
Define the ascent-step parameter
\begin{align}
w_{t+1/2} = w_t + r\,\nabla L_t(w_t)_{\Omega}.
\end{align}
By $\beta$-smoothness,
\begin{align}
L(w_{t+1})
\le 
L(w_t)
- \eta\,\langle \nabla L_t(w_{t+1/2}), \nabla L(w_t)\rangle
+ \frac{\eta^2\beta}{2}\|\nabla L_t(w_{t+1/2})\|^2.
\end{align}
Taking expectations and using 
$\langle p,q\rangle = \tfrac12(\|p\|^2 + \|q\|^2 - \|p-q\|^2)$
with $p = \nabla L(w_{t+1/2})$ and $q = \nabla L(w_t)$ (and inserting/subtracting the population gradient where needed), we obtain
\begin{align}
\mathbb{E}[L(w_{t+1})]
&\le 
\mathbb{E}[L(w_t)]
- \frac{\eta}{2}\Big(\mathbb{E}\|\nabla L(w_{t+1/2})\|^2 
+ \mathbb{E}\|\nabla L(w_t)\|^2 - E_1\Big)
+ \frac{\eta^2\beta}{2}E_2,
\end{align}
where
\begin{align}
E_1 &= \mathbb{E}\big\|\nabla L(w_{t+1/2}) - \nabla L(w_t)\big\|^2,\\
E_2 &= \mathbb{E}\big\|\nabla L_t(w_{t+1/2})\big\|^2.
\end{align}

\paragraph{Bounding $E_1$.}
By $\beta$-smoothness,
\begin{align}
\|\nabla L(w_{t+1/2}) - \nabla L(w_t)\|^2
&\le \beta^2 \|w_{t+1/2} - w_t\|^2
= \beta^2 r^2 \|\nabla L_t(w_t)_{\Omega}\|^2.
\end{align}
Hence
\begin{align}
E_1
&\le \beta^2 r^2\,\mathbb{E}\|\nabla L_t(w_t)_{\Omega}\|^2 \\
&= \beta^2 r^2\,\mathbb{E}\big\|\nabla L_t(w_t)_{\Omega} - \nabla L(w_t)_{\Omega} + \nabla L(w_t)_{\Omega}\big\|^2 \\
&\le 2\beta^2 r^2\,\mathbb{E}\|\nabla L_t(w_t)_{\Omega} - \nabla L(w_t)_{\Omega}\|^2
   + 2\beta^2 r^2\,\mathbb{E}\|\nabla L(w_t)_{\Omega}\|^2 \\
&\le \frac{2\beta^2 r^2}{b}\sigma_{\Omega}^2
   + 2\beta^2 r^2\,\mathbb{E}\|\nabla L(w_t)\|^2,
\end{align}
where we used $\|\nabla L(w_t)_{\Omega}\|^2 \le \|\nabla L(w_t)\|^2$ since filtering only removes coordinates.

\paragraph{Bounding $E_2$.}
We have
\begin{align}
E_2
&= \mathbb{E}\|\nabla L_t(w_{t+1/2})\|^2 \\
&\le 
2\,\mathbb{E}\|\nabla L_t(w_{t+1/2}) - \nabla L(w_{t+1/2})\|^2
+ 2\,\mathbb{E}\|\nabla L(w_{t+1/2})\|^2 \\
&\le 
\frac{2\sigma_{\Omega}^2}{b}
+ 2\,\mathbb{E}\|\nabla L(w_{t+1/2})\|^2.
\end{align}
Using smoothness again with 
$u = w_t + r\nabla L(w_t)_{\Omega}$,
\begin{align}
\|\nabla L(w_{t+1/2}) - \nabla L(u)\|
&\le \beta r\,\|\nabla L_t(w_t)_{\Omega} - \nabla L(w_t)_{\Omega}\|,
\end{align}
which implies
\begin{align}
\|\nabla L(w_{t+1/2})\|^2
&\le 2\|\nabla L(u)\|^2 
  + 2\beta^2 r^2\,\|\nabla L_t(w_t)_{\Omega} - \nabla L(w_t)_{\Omega}\|^2.
\end{align}
Taking expectations and applying the variance bound,
\begin{align}
\mathbb{E}\|\nabla L(w_{t+1/2})\|^2
&\le 2\|\nabla L(w_t + r\nabla L(w_t)_{\Omega})\|^2
  + \frac{2\beta^2 r^2}{b}\sigma_{\Omega}^2.
\end{align}
Thus
\begin{align}
E_2
&\le 
\frac{2\sigma_{\Omega}^2}{b}
+ 2\|\nabla L(w_t + r\nabla L(w_t)_{\Omega})\|^2
+ \frac{2\beta^2 r^2}{b}\sigma_{\Omega}^2.
\end{align}

\paragraph{Putting the bounds together.}
Substituting the bounds on $E_1$ and $E_2$ into the main inequality gives
\begin{align}
\mathbb{E}[L(w_{t+1})]
&\le 
\mathbb{E}[L(w_t)]
- \frac{\eta}{2}\mathbb{E}\|\nabla L(w_{t+1/2})\|^2
- \frac{\eta}{2}\mathbb{E}\|\nabla L(w_t)\|^2
+ \frac{\eta}{2}E_1
+ \frac{\eta^2\beta}{2}E_2 \\
&\le 
\mathbb{E}[L(w_t)]
- \frac{\eta}{2}\mathbb{E}\|\nabla L(w_{t+1/2})\|^2
- \eta\Big(\tfrac12 - \beta^2 r^2\Big)\mathbb{E}\|\nabla L(w_t)\|^2 \\
&\quad
+ \frac{\eta\beta^2 r^2}{b}\sigma_{\Omega}^2
+ \frac{\eta^2\beta}{b}\sigma_{\Omega}^2
+ \eta^2\beta\,\mathbb{E}\|\nabla L(w_{t+1/2})\|^2.
\end{align}
The coefficient of $\mathbb{E}\|\nabla L(w_{t+1/2})\|^2$ is
\begin{align}
-\frac{\eta}{2} + \eta^2\beta.
\end{align}
For $\eta \le \tfrac{1}{4\beta}$ this coefficient is non-positive, so we can drop this term.  
Moreover, if $\beta^2 r^2 \le \tfrac14$, then
\(\tfrac12 - \beta^2 r^2 \ge \tfrac14\), and hence
\begin{align}
\mathbb{E}[L(w_{t+1})]
&\le 
\mathbb{E}[L(w_t)]
- \frac{\eta}{4}\,\mathbb{E}\|\nabla L(w_t)\|^2
+ \frac{2\eta\beta^2 r^2}{b}\sigma_{\Omega}^2
+ \frac{\eta^2\beta}{b}\sigma_{\Omega}^2,
\end{align}
which proves the claim.
\end{proof}

\begin{theorem}
Assume that $L$ is $\beta$-smooth and that the filtered stochastic gradient
$\nabla L_t(w_t)_{\Omega}$ satisfies the variance bound
\begin{align}
\mathbb{E}\!\left[\|\nabla L_t(w_t)_{\Omega} - \nabla L(w_t)_{\Omega}\|^2\right]
\le \frac{\sigma_{\Omega}^2}{b}.
\end{align}
All expectations are taken over the mini-batch at iteration $t$.
Suppose the ZSharp-SAM updates
\begin{align}
w_{t+1/2} &= w_t + r \,\nabla L_t(w_t)_{\Omega}, \\
w_{t+1}   &= w_{t+1/2} - \eta \,\nabla L_t(w_{t+1/2})
\end{align}
use step size $\eta$ and ascent radius $r$ satisfying
$\eta \le \tfrac{1}{4\beta}$ and $\beta^2 r^2 \le \tfrac{1}{4}$.
Then ZSharp-SAM satisfies
\begin{align}
\frac{1}{T} \sum_{t=0}^{T-1} \mathbb{E} \big[ \|\nabla L(w_t)\|^2 \big]
\le \frac{4}{T \eta} \big( L(w_0) - \mathbb{E}[L(w_T)] \big)
+ \frac{8\beta^2 r^2}{b}\,\sigma_{\Omega}^2
+ \frac{4\eta \beta}{b}\,\sigma_{\Omega}^2.
\end{align}
\end{theorem}

\begin{proof}
From Lemma~4 (ZSharp-SAM one-step descent bound), for each $t$ we have
\begin{align}
\mathbb{E}[L(w_{t+1})]
&\le \mathbb{E}[L(w_t)]
- \frac{\eta}{4}\,\mathbb{E}\big[\|\nabla L(w_t)\|^2\big]
+ \frac{2\eta \beta^2 r^2}{b}\,\sigma_{\Omega}^2
+ \frac{\eta^2 \beta}{b}\,\sigma_{\Omega}^2.
\label{eq:a5_step_new}
\end{align}
Averaging \eqref{eq:a5_step_new} over $t = 0,\dots,T-1$ yields
\begin{align}
\frac{1}{T} \sum_{t=0}^{T-1} \mathbb{E}[L(w_{t+1})]
&\le \frac{1}{T} \sum_{t=0}^{T-1} \mathbb{E}[L(w_t)]
- \frac{\eta}{4T} \sum_{t=0}^{T-1} \mathbb{E}\big[\|\nabla L(w_t)\|^2\big]
+ \frac{2\eta \beta^2 r^2}{b}\,\sigma_{\Omega}^2
+ \frac{\eta^2 \beta}{b}\,\sigma_{\Omega}^2.
\label{eq:a5_avg_new}
\end{align}
Using the telescoping identity
\begin{align}
\frac{1}{T} \sum_{t=0}^{T-1}
\big(\mathbb{E}[L(w_t)] - \mathbb{E}[L(w_{t+1})]\big)
= \frac{1}{T} \big(L(w_0) - \mathbb{E}[L(w_T)]\big),
\end{align}
inequality \eqref{eq:a5_avg_new} can be rewritten as
\begin{align}
\frac{\eta}{4T} \sum_{t=0}^{T-1} \mathbb{E}\big[\|\nabla L(w_t)\|^2\big]
&\le \frac{1}{T}\big(L(w_0) - \mathbb{E}[L(w_T)]\big)
+ \frac{2\eta \beta^2 r^2}{b}\,\sigma_{\Omega}^2
+ \frac{\eta^2 \beta}{b}\,\sigma_{\Omega}^2.
\label{eq:a5_before_div_new}
\end{align}
Dividing both sides of \eqref{eq:a5_before_div_new} by $\eta/4$ gives
\begin{align}
\frac{1}{T} \sum_{t=0}^{T-1} \mathbb{E}\big[\|\nabla L(w_t)\|^2\big]
&\le \frac{4}{T \eta} \big(L(w_0) - \mathbb{E}[L(w_T)]\big)
+ \frac{8\beta^2 r^2}{b}\,\sigma_{\Omega}^2
+ \frac{4\eta \beta}{b}\,\sigma_{\Omega}^2,
\end{align}
which is the claimed bound.
\end{proof}

\begin{corollary}[Convergence with diminishing step sizes]
Assume that $L$ is $\beta$-smooth, bounded below by $L_{\inf}$, and $\nabla L_t(w_t)_{\Omega}$ satisfies $\mathbb{E}\!\left[\|\nabla L_t(w_t)_{\Omega} - \nabla L(w_t)_{\Omega}\|^2\right]
$$\le \frac{\sigma_{\Omega}^2}{b}.$
Consider the ZSharp-SAM updates $w_{t+1/2} = w_t + r_t \,\nabla L_t(w_t)_{\Omega}, $ and $w_{t+1}   = w_{t+1/2} - \eta_t \,\nabla L_t(w_{t+1/2}),$ with step sizes $\{\eta_t\}_{t\ge 0}$ and ascent radii $\{r_t\}_{t\ge 0}$ satisfying
$\eta_t \le \frac{1}{4\beta} $$ \beta^2 r_t^2 \le \frac{1}{4} \text{for all } t,$
$\sum_{t=0}^\infty \eta_t = \infty, $$\sum_{t=0}^\infty \eta_t^2 < \infty, $$\sum_{t=0}^\infty \eta_t r_t^2 < \infty.$
Then
\begin{align}
\sum_{t=0}^\infty \eta_t\,\mathbb{E}\big[\|\nabla L(w_t)\|^2\big] < \infty,
\end{align}
and in particular
\begin{align}
\liminf_{t\to\infty} \mathbb{E}\big[\|\nabla L(w_t)\|^2\big] = 0.
\end{align}
\end{corollary}

\begin{proof}
By the same argument as in Lemma~4, but allowing time-varying $(\eta_t,r_t)$,
for each $t$ we obtain
\begin{align}
\mathbb{E}[L(w_{t+1})]
\le 
\mathbb{E}[L(w_t)]
- \frac{\eta_t}{4}\,\mathbb{E}\big[\|\nabla L(w_t)\|^2\big]
+ \frac{2\eta_t \beta^2 r_t^2}{b}\,\sigma_{\Omega}^2
+ \frac{\eta_t^2 \beta}{b}\,\sigma_{\Omega}^2.
\label{eq:cor_step}
\end{align}
Rearranging,
\begin{align}
\frac{\eta_t}{4}\,\mathbb{E}\big[\|\nabla L(w_t)\|^2\big]
&\le 
\mathbb{E}[L(w_t)] - \mathbb{E}[L(w_{t+1})]
+ \frac{2\eta_t \beta^2 r_t^2}{b}\,\sigma_{\Omega}^2
+ \frac{\eta_t^2 \beta}{b}\,\sigma_{\Omega}^2.
\end{align}
Summing from $t=0$ to $T-1$ gives
\begin{align}
\frac{1}{4} \sum_{t=0}^{T-1} \eta_t\,\mathbb{E}\big[\|\nabla L(w_t)\|^2\big]
&\le 
\mathbb{E}[L(w_0)] - \mathbb{E}[L(w_T)]
+ \frac{2\beta^2\sigma_{\Omega}^2}{b} \sum_{t=0}^{T-1} \eta_t r_t^2
+ \frac{\beta \sigma_{\Omega}^2}{b} \sum_{t=0}^{T-1} \eta_t^2.
\end{align}
Using $L(w_T) \ge L_{\inf}$ and letting $T\to\infty$, the right-hand side is bounded by
\begin{align}
\mathbb{E}[L(w_0)] - L_{\inf}
+ \frac{2\beta^2\sigma_{\Omega}^2}{b} \sum_{t=0}^{\infty} \eta_t r_t^2
+ \frac{\beta \sigma_{\Omega}^2}{b} \sum_{t=0}^{\infty} \eta_t^2 < \infty,
\end{align}
by the assumptions on $\{\eta_t\}$ and $\{r_t\}$. Hence
\begin{align}
\sum_{t=0}^{\infty} \eta_t\,\mathbb{E}\big[\|\nabla L(w_t)\|^2\big] < \infty.
\end{align}
Because $\sum_{t=0}^{\infty} \eta_t = \infty$, this implies
\begin{align}
\liminf_{t\to\infty} \mathbb{E}\big[\|\nabla L(w_t)\|^2\big] = 0,
\end{align}
which completes the proof.
\end{proof}

\section{Experimental Results}

\begin{figure*}[h]
\includegraphics[width=1.005\textwidth]{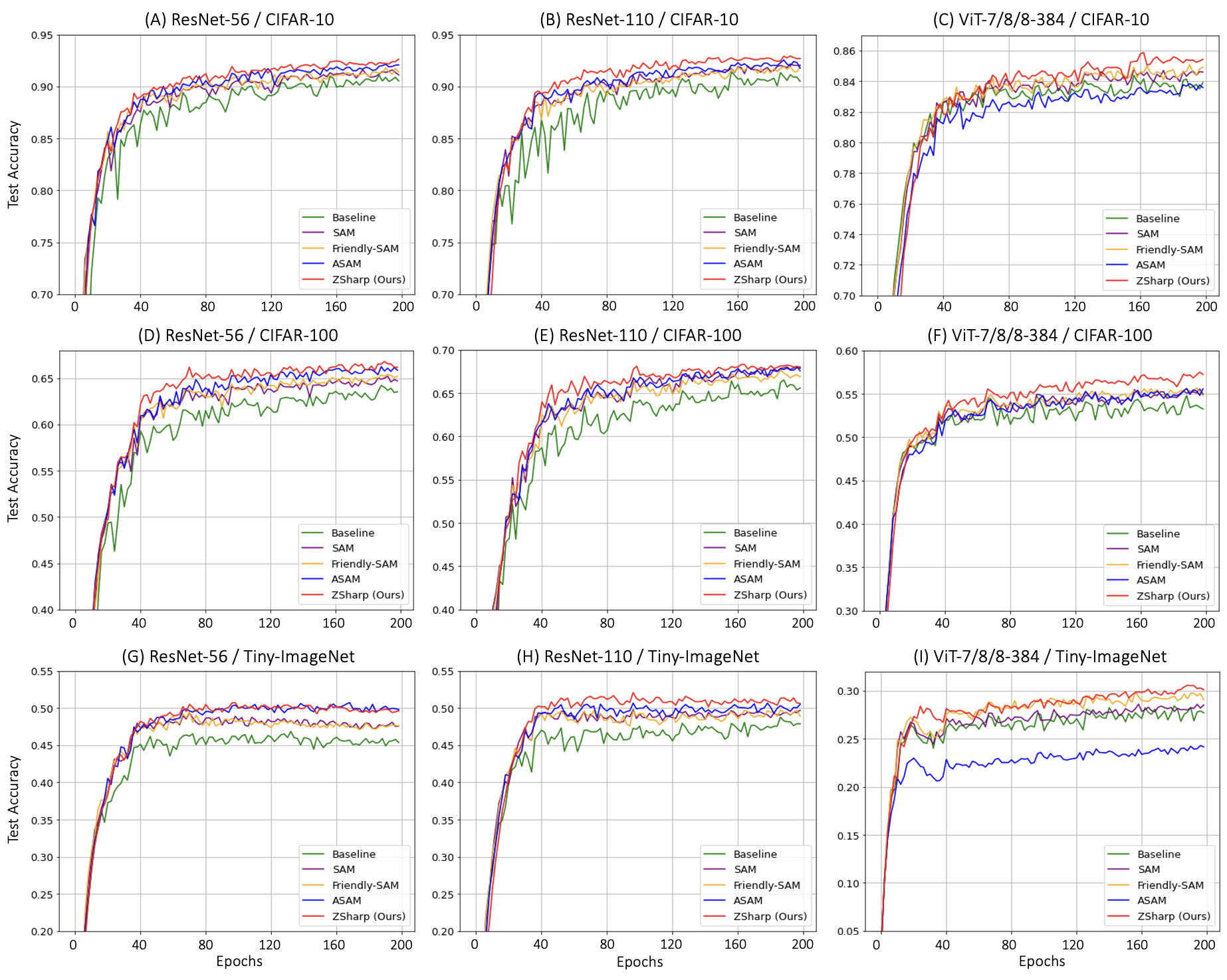}
\vspace{-15pt}
\caption{Top-1 Test Accuracy comparison on CIFAR-10 for ResNet-56/110 and ViT-7/8/8-384 models across different SAM variants: AdamW (Baseline)~\cite{adamw}, SAM~\cite{foret2021sharpness}, Friendly-SAM~\cite{kim2022friendly}, ASAM~\cite{kwon2021asam}, and ZSharp (Ours). The red dashed line indicates the baseline performance using AdamW~\cite{adamw} alone, highlighting the improvements achieved by sharpness-aware methods. ZSharp consistently outperforms other methods, demonstrating the effectiveness of ZNorm-based gradient filtering in enhancing generalization.}
\label{fig:allgraph}
\end{figure*}

\noindent \textbf{Experimental Settings. }
We evaluate ZSharp on three standard benchmarks: CIFAR-10~\cite{krizhevsky2009learning}, CIFAR-100~\cite{krizhevsky2009learning}, and Tiny‑ImageNet~\cite{tinyimagenet}. CIFAR-10 and CIFAR-100 each consist of 50,000 training and 10,000 test images at 32$\times$32 resolution across 10 and 100 classes, respectively. Tiny‑Image Net contains 90,000 training and 10,000 test images of size 64$\times$64 across 200 classes. We benchmark ZSharp using ResNet-56/110~\cite{he2016deep}, VGG16\_BN~\cite{simonyan2015very}. For ViT models, ViT-7/8/8-384 and ViT-7/8/12-768 denote Vision Transformers~\cite{dosovitskiy2020image} with 7 layers, 8 attention heads, patch sizes of 8, and MLP dimensions of 384 and 768, respectively. All trained without pre-trained weights. All models are trained for 200 epochs with a batch size of 256 using the AdamW optimizer~\cite{kingma2015adam, adamw} (initial learning rate 0.001, weight decay $5\times10^{-5}$). The learning rate is scheduled using a step decay policy, where it is multiplied by 0.75 every 10 epochs. For SAM~\cite{foret2021sharpness}, ASAM~\cite{kwon2021asam}, Friendly-SAM~\cite{kim2022friendly} and ZSharp, the perturbation radius $\rho$ is set to 0.05 as used in each paper, and we followed optimal hyperparameters written in papers. ZSharp uses Z-score filtering with $Q_p = 0.95$, retaining only the top 5\% of gradient components during the ascent step. All experiments are conducted on a single NVIDIA RTX 4090 GPU, and results are averaged over 3 different random seeds to ensure statistical robustness.

\subsection{Experimental Results}
\label{figures}
Figure~\ref{fig:allgraph} compares the test accuracy over training epochs for Baseline (AdamW), SAM, Friendly-SAM, ASAM, and the proposed ZSharp across different architectures and datasets.

Across all settings, ZSharp consistently achieves the highest or near-highest test accuracy, showing faster convergence and better final performance compared to other methods. In ResNet-56 and ResNet-110 on CIFAR-10 (A, B), ZSharp reaches higher accuracy earlier and maintains a margin over SAM variants. Similar improvements are observed for ViT-7/8/8-384 on CIFAR-10 (C), where ZSharp outperforms other methods throughout most of the training.

For CIFAR-100 (D–F), the advantage of ZSharp becomes more evident, with a clear gap over SAM and Friendly-SAM, and competitive or better performance compared to ASAM. The pattern persists on Tiny-ImageNet (G–I), where ZSharp not only surpasses the baseline and SAM variants but also shows greater stability, particularly for transformer architectures (I), which exhibit a larger improvement margin. These results demonstrate that applying Z-score gradient filtering in the ascent step of SAM can enhance generalization performance consistently across datasets and architectures.

\subsection{Hyperparameter Tuning}

\begin{table*}[h]
\caption{Training hyperparameters and results for all experiments, with settings identical to those in the Experimental Settings section except for varying \(Q_p\) values. Models include ResNet-56~\cite{he2016deep} and ViT-7/8/8-384~\cite{dosovitskiy2020image}, where ViT-7/8/384 denotes a Vision Transformer with 7 layers, 8 attention heads, a patch size of 8, and an MLP dimension of 384, all trained without pre-trained weights.}
\centering
\begin{adjustbox}{width=1\textwidth}
\begin{tabular}{lllll|lllll} 
\toprule
Network & Method & \(Q_p\) & Top-1 Test Acc. & Train Loss. & Network & Method & \(Q_p\) & Top-1 Test Acc. & Train Loss. \\
\midrule
\multirow{7}{*}{ResNet-56~\cite{he2016deep}}  & AdamW~\cite{adamw}               & N/A   & 0.9108 ± 0.0045 & 0.0057 ± 0.0013            & \multirow{7}{*}{ViT-7/8/8-384~\cite{dosovitskiy2020image}}  & AdamW~\cite{adamw}     & N/A  & 0.8398 ± 0.0028 & 0.0087 ± 0.0092\\
& SAM~\cite{foret2021sharpness}              & N/A   & 0.9160 ± 0.0021 & 0.0221 ± 0.0051            && SAM~\cite{foret2021sharpness}   & N/A   & 0.8432 ± 0.0032 & 0.0273 ± 0.0101\\
& ZSharp (Ours)                              & 0.95  & \textbf{0.9264 ± 0.0032} & 0.0630 ± 0.0064   && ZSharp (Ours) & 0.95  & \textbf{0.8543 ± 0.0029} & 0.0647 ± 0.0216\\
& ZSharp (Ours)                              & 0.90  & 0.9212 ± 0.0015 & 0.0710 ± 0.0061            && ZSharp (Ours) & 0.90  & 0.8482 ± 0.0031 & 0.0748 ± 0.0081\\
& ZSharp (Ours)                              & 0.85  & 0.9189 ± 0.0023 & 0.0679 ± 0.0067            && ZSharp (Ours) & 0.85  & 0.8424 ± 0.0043 & 0.0825 ± 0.0086\\
& ZSharp (Ours)                              & 0.80  & 0.9153 ± 0.0027 & 0.0731 ± 0.0053            && ZSharp (Ours) & 0.80  & 0.8421 ± 0.0038 & 0.0863 ± 0.0119\\
& ZSharp (Ours)                              & 0.75  & 0.9132 ± 0.0017 & 0.0789 ± 0.0079            && ZSharp (Ours) & 0.75  & 0.8378 ± 0.0027 & 0.0999 ± 0.0102\\
\bottomrule
\end{tabular}
\end{adjustbox}
\label{table:training_settings}
\end{table*}

\label{hyperparameter}
We evaluate the effect of the percentile threshold \(Q_p\) in ZSharp, which controls the proportion of gradient components retained after ZNorm filtering, selecting the top \((1-Q_p)\%\) based on Z-scores. A higher \(Q_p\) (e.g., 0.95) retains fewer components (top 5\%), focusing on significant directions for sharpness-aware optimization, while a lower \(Q_p\) (e.g., 0.75) retains more (top 25\%). At \(Q_p = 0.0\), ZSharp reduces to SAM, using the full gradient. We test \(Q_p \in \{0.75, 0.80, 0.85, 0.90, 0.95\}\) on ResNet-56 and ViT-7/8/8-384, with results in Table~\ref{table:training_settings}.

Table~\ref{table:training_settings} shows ZSharp achieves the highest Top-1 Test Accuracy at \(Q_p = 0.95\), with 0.9264 ± 0.0032 on ResNet-56~\cite{he2016deep} and 0.8543 ± 0.0029 on ViT-7/8/8-384~\cite{dosovitskiy2020image}, outperforming AdamW~\cite{adamw} and SAM~\cite{foret2021sharpness}. As \(Q_p\) decreases to 0.75, test accuracy nears SAM’s (e.g., 0.9132 on ResNet-56), reflecting ZSharp’s alignment with SAM’s behavior. Based on these results, we identify \(Q_p = 0.95\) as the optimal value and use it for all subsequent experiments to maximize generalization performance.


\end{document}